\documentclass[twoside]{article}

\usepackage{amsmath,amsthm,amssymb}

\usepackage{color}
\definecolor{Blue}{rgb}{0.9,0.3,0.3}

\newcommand{\squishlist}{
   \begin{list}{$\bullet$}
    { \setlength{\itemsep}{0pt}      \setlength{\parsep}{3pt}
      \setlength{\topsep}{3pt}       \setlength{\partopsep}{0pt}
      \setlength{\leftmargin}{1.5em} \setlength{\labelwidth}{1em}
      \setlength{\labelsep}{0.5em} } }

\newcommand{\squishlisttwo}{
   \begin{list}{$\bullet$}
    { \setlength{\itemsep}{0pt}    \setlength{\parsep}{0pt}
      \setlength{\topsep}{0pt}     \setlength{\partopsep}{0pt}
      \setlength{\leftmargin}{2em} \setlength{\labelwidth}{1.5em}
      \setlength{\labelsep}{0.5em} } }

\newcommand{\squishend}{
    \end{list}  }

\newcommand{\myvec}[1]{\mathbf{#1}}

\newcommand{\vf}{\myvec{f}}

\newcommand{\vk}{\myvec{k}}

\newcommand{\vm}{\myvec{m}}

\newcommand{\vx}{\myvec{x}}

\newcommand{\vy}{\myvec{y}}

\newcommand{\vD}{\myvec{D}}

\newcommand{\vK}{\myvec{K}}

\newcommand{\calB}{{\cal B}}

\newcommand{\calD}{{\cal D}}

\newcommand{\calH}{{\cal H}}

\newcommand{\calN}{{\cal N}}

\newcommand{\calL}{{\cal L}}

\newcommand{\calT}{{\cal T}}

\newcommand{\calX}{{\cal X}}

\newcommand{\calU}{{\cal U}}

\newcommand{\data}{\calD}

\newcommand{\be}{\begin{equation}}
\newcommand{\ee}{\end{equation}}
\newcommand{\bea}{\begin{eqnarray}}
\newcommand{\eea}{\end{eqnarray}}
\newcommand{\beaa}{\begin{eqnarray*}}
\newcommand{\eeaa}{\end{eqnarray*}}

\DeclareMathAlphabet{\mathpzc}{OT1}{pzc}{m}{n}

\usepackage{url}
\usepackage{graphicx}
\usepackage{algorithm}
\usepackage{algorithmic}
\usepackage{times}
\usepackage{flushend} 
\usepackage{natbib}
\usepackage{fancyhdr}
\usepackage{color}

\newtheorem{mydefinition}{Definition}
\newtheorem{mydefinition1}{Definition}
\newtheorem{mydefinition2}{Definition}

\newtheorem{mydefinition4}{Definition}
\newtheorem{mydefinition5}{Definition}

\newtheorem{lemma}[mydefinition1]{Lemma}
\newtheorem{theorem}[mydefinition2]{Theorem}

\newtheorem{corollary}[mydefinition4]{Corollary}
\newtheorem{assumption}[mydefinition5]{Assumption}

\theoremstyle{remark}

\DeclareMathOperator*{\argmax}{arg\,max}
\DeclareMathOperator*{\Span}{span}
\DeclareMathOperator*{\indicator}{ \mathbf{1} }

\makeatletter
\newcommand{\thickhline}{%
    \noalign {\ifnum 0=`}\fi \hrule height 1pt
    \futurelet \reserved@a \@xhline
}

\usepackage[accepted]{aistats2014}

\begin{document}

\twocolumn[

\aistatstitle{Bayesian Multi-Scale Optimistic Optimization}

\aistatsauthor{Ziyu Wang \And  Babak Shakibi \And  Lin Jin \And  Nando de Freitas} 
\aistatsaddress{ University of Oxford \And University of British Columbia \And Rocket Gaming Systems \And University of Oxford}
]

\begin{abstract}

Bayesian optimization is a powerful global optimization technique
for expensive black-box functions.
One of its shortcomings is that it requires auxiliary optimization
of an acquisition function at each iteration.
This auxiliary optimization can be costly and very hard to carry out in practice.
Moreover, it creates serious theoretical concerns,
as most of the convergence results assume that the exact optimum of
the acquisition function can be found.
In this paper, we introduce a new technique for efficient global optimization
that combines Gaussian process confidence bounds and
treed simultaneous optimistic optimization to eliminate the need
for auxiliary optimization of acquisition functions.
The experiments with global optimization benchmarks and a novel application to automatic information extraction demonstrate that the resulting technique is more efficient
than the two approaches from which it draws inspiration.
Unlike most theoretical analyses of Bayesian optimization with Gaussian processes, our finite-time convergence rate proofs do not require exact optimization of an acquisition function. That is, our approach eliminates the unsatisfactory assumption that a difficult, potentially NP-hard, problem has to be solved in order to obtain vanishing regret rates.
\end{abstract}

\section{Introduction}
We consider the problem of approximating the maximizer of a deterministic black-box function
$f: {\cal X} \mapsto \mathbb{R}$. The function $f$ can be evaluated point-wise, but it is assumed to be expensive to evaluate.
More precisely, we assume that we are given a finite budget of $n$ possible function evaluations.

This global optimization problem can be treated within the framework of sequential design.
In this context, by allowing $\vx_t \in {\cal X}$ to depend on previous points
and corresponding function evaluations
$\data_{t-1} = \{(\vx_1,f(\vx_1)), \ldots, (\vx_{t-1},f(\vx_{t-1}))\}$,
the algorithm constructs a sequence $\vx_{1:n} = (\vx_1, \vx_2, \ldots, \vx_n)$
and returns the element $\vx(n)$ of highest possible value.
That is, it returns the value $\vx(n)$ that minimizes the loss:
\[
r_n = \sup_{\vx\in {\cal X}}f(\vx) - f(\vx(n)).
\]
This loss is not the same as the cumulative regret used often in
the online learning literature:
$
R_n = n \sup_{\vx\in {\cal X}}f(\vx) - \sum_{t=1}^{n} f(\vx(t)).
$

\emph{Bayesian optimization} (BO) is a popular sequential design strategy
for global optimization; see \cite{Brochu:2009} for an introductory treatment.
Since the objective function $f$ is unknown, the Bayesian strategy is to treat it as a random function and place a prior over it.
The prior captures our beliefs about the behaviour of the function.
After gathering the function evaluations $\data_{t-1}$, the prior is updated
to form the posterior distribution over $f$.
The posterior distribution, in turn, is used to construct
an \emph{acquisition function} that determines what the next query point $\vx_t$ should be.
Examples of acquisition functions include probability of improvement,
expected improvement, Bayesian expected losses, upper confidence bounds (UCB),
and dynamic portfolios of these
\citep{Mockus:1982,Jones:2001,Garnett:2010,Srinivas:2010,Chen:2012,Hoffman:2011}. If we were to implement Thompson sampling strategies \citep{May:2011,Kaufmann:2012,Agrawal:2013} for Gaussian processes (GPs), we would also encounter the difficult problem of having to find the maximizer of a sample from the GP at each iteration, unless we were considering only a finite set of query points \citep{Hoffman:2014}.

The maximum of the acquisition function is typically found by resorting to
discretisation or by means of an \emph{auxiliary optimizer}.
For example, \cite{Snoek:2012} use discretisation, \cite{Bardenet:2010}
use adaptive grids, \cite{Brochu:2007,martinez-cantin:2007} and \cite{Mahendran:2012}
use the DIRECT algorithm of \cite{Jones:1993}, \cite{Lizotte:2011}
use a combination of random discretisation and quasi-Newton hill-climbing,
\cite{Bergstra:2011} and \cite{Wang:rembo} use the CMA-ES method of \cite{Hansen:2001},
\cite{Hutter:smac} apply multi-start local search. (Approaches within the framework of Bayesian nonlinear experimental design, such as
\citep{Hennig:2012} for finding maxima and \citep{Kueck:2006,Kueck:2009,Hoffman:2009} for learning functions and Markov decision processes, have to rely on expensive approximate inference for computing intractable integrals. An analysis of these approaches is beyond the scope of this paper.)

The auxiliary optimization methodology is problematic for several reasons.
First, it is difficult to assess whether the auxiliary optimizer
has found the maximum of the acquisition function in practice.
This creates important theoretical concerns about the behaviour of BO algorithms
because the typical theoretical convergence guarantees are only valid
on the assumption that the optimum of the acquisition function can be found exactly;
see for example \cite{Srinivas:2010,Vazquez:2011} and
\cite{Bull:2011}.
Second, running an auxiliary optimizer at each iteration of the BO algorithm
can be unnecessarily costly. For any two
consecutive iterations, the acquisition function may not change drastically.
This questions the necessity of re-starting the auxiliary optimization at each iteration.

Recent \emph{optimistic optimization} methods provide a viable alternative
to BO~\citep{Kocsis:2006, Bubeck:2011, Munos:2011}.
Instead of estimating a posterior distribution over the unknown objective function,
these methods build space partitioning trees by expanding leaves
with high function values or upper-bounds.
The term optimistic, in this context, is used to refer to the fact
that the algorithms expand at each round leaves that may contain the optimum.
Remarkably, a variant of these methods,
\emph{Simultaneous Optimistic Optimization} (SOO) by \cite{Munos:2011},
is able to optimize an objective function globally without knowledge of the function's
smoothness. SOO is optimistic at all scales in the sense that it expands several leaves simultaneously, with at most one leaf per level. For this reason, instead of adopting the term ``Simultaneous OO'' we opt for the descriptive term ``Multi-Scale OO''.

We will describe SOO in more detail in Section~\ref{sec:soo}.
We also note that a stochastic variant of SOO
has been recently proposed by \cite{Valko:SSOO},
but we restrict the focus of this paper to the deterministic case.

These optimistic optimization methods
do not require the auxiliary optimization of
acquisition functions. However, due to the lack of a posterior that interpolates
between the sampled points,
it is conceivable that these methods may not be as competitive
as BO in practical domains where prior knowledge is available.
This claim does not seem to have been backed up by empirical evidence in the past.

This paper introduces a new algorithm, BaMSOO,
which combines elements of BO and SOO.
Importantly, it eliminates the need for auxiliary optimization
of the acquisition function in BO.
We derive theoretical guarantees for the method that do not depend
on the assumption that the acquisition function needs to be optimized exactly.
The method uses SOO to optimize the objective function directly,
but eliminates the need for SOO to sample
points that are deemed unfit by Gaussian process posterior bounds.
That is, BaMSOO uses the posterior distribution to reduce
the number of function evaluations in SOO,
thus increasing the efficiency of SOO substantially.

The experiments with benchmarks from the global optimization literature demonstrate that BaMSOO outperforms both GP-UCB and SOO. The paper also introduces a novel application in the domain of knowledge discovery and information extraction. Finally,
our theoretical results show that BaMSOO can attain,
up to log factors, a polynomial finite sample convergence rate.

\section{BO with GP confidence bounds}

Classical BO approaches have two ingredients that
need to be specified: The prior and the acquisition function.
In this work, as in most other works, we adopt Gaussian process (GP) priors.
We review GPs very briefly and refer the
interested reader to the book of \cite{Rasmussen:2006} for an in-depth treatment.
A GP is a distribution over functions specified by its mean
function $m(\cdot)$ and covariance $\kappa(\cdot,\cdot)$.
More specifically, given a set of points $\vx_{1:t}$,
with $\vx_i \in {\cal X} \subseteq \mathbb{R}^D$, we have
$$
\vf(\vx_{1:t}) \sim \mathcal{N}(\vm(\vx_{1:t}), \vK),
$$
where $\vK$, with entries $\vK_{i,j} = \kappa(\vx_i, \vx_j)$,
is the covariance matrix.
A common choice of $\kappa$ in the BO literature is the
anisotropic kernel with a vector of known hyper-parameters
\bea
\kappa(\vx_i, \vx_j) &=& \widetilde{\kappa}\left(-(\vx_i-\vx_j)^T\vD(\vx_i-\vx_j)\right),
\eea
where $\widetilde{\kappa}$ is an isotropic kernel and $\vD$
is a diagonal matrix with positive hyper-parameters
along the diagonal and zeros elsewhere.
Our results apply to squared exponential kernels
and Mat\'ern kernels with parameter $\nu \geq 2$.
In this paper, we assume that the hyper-parameters are fixed and known in advance.
We refer the reader to \cite{martinez-cantin:2007,Brochu:2010,Wang:rembo,Snoek:2012} for different practical approaches to estimate the hyper-parameters.

An advantage of using GPs lies in their analytical tractability.
In particular, given observations $\data_t = \{\vx_{1:t}, \vf_{1:t} \}$,
where $f_i = f(\vx_i)$,  and a new point $\vx_{t+1}$,
the joint distribution is given by:
$$
\begin{bmatrix}\vf_{1:t} \\
f_{t+1} \end{bmatrix} \sim \mathcal{N}\left( \vm(\vx_{1:t+1}),  \begin{bmatrix}
\vK & \vk\\
\vk^{T} & \kappa(\vx_{t+1}, \vx_{t+1})\end{bmatrix}\right)
$$
where $\vk^T = [\kappa(\vx_{t+1},\vx_1) \cdots \kappa(\vx_{t+1},\vx_t)]$.
For simplicity, we assume that $\vm(\cdot) = \mathbf{0}$.
Using the Sherman-Morrison-Woodbury formula,
one can easily arrive at the posterior predictive distribution:
$$
f_{t+1} | \data_t, \vx_{t+1} \sim \mathcal{N}(\mu(\vx_{t+1}|\data_t),
\sigma^2(\vx_{t+1}|\data_t)),
$$
with mean
$\mu(\vx_{t+1}|\data_t) = \vk^{T} \vK^{-1} \vf_{1:t}$
and variance
$
 \sigma^2(\vx_{t+1}|  \data_t) =
\kappa(\vx_{t+1}, \vx_{t+1})-  \vk^{T} \vK^{-1} \vk.$
We can compute the posterior predictive mean $\mu(\cdot)$
and variance $\sigma^2(\cdot)$ exactly for any point $\vx_{t+1}$.

At each iteration of BO, one has to re-compute
the predictive mean and variance.
These two quantities are used to construct the second ingredient of
BO: The acquisition function (or utility function).
In this work, we report results for the GP-UCB acquisition function
$\calU(\vx|\data_t) = \mu(\vx|\data_t) + \sqrt{B_t} \sigma(\vx|\data_t)$,
which is the upper confidence bound (UCB)
on the objective function \citep{Srinivas:2010,deFreitas:2012}.
We also make use of the lower confidence bound (LCB)
which is defined as
$\calL(\vx|\data_t) = \mu(\vx|\data_t) - \sqrt{B_t} \sigma(\vx|\data_t)$.
In these definitions, $B_t$ is such that
$f(\vx)$ is bounded above and below by $\calU(\vx|\data_t)$ and
$\calL(\vx|\data_t)$ with high probability \citep{deFreitas:2012}.

BO selects the next query point by optimizing the
acquisition function $\calU(\vx|\mathcal{D}_t).$
Note that our choice of utility favours the selection of points with high variance
(points in regions not well explored)
and points with high mean value (points worth exploiting).
As mentioned in the introduction, the optimization of the closed-form
acquisition function is often carried out by
off-the-shelf global optimization procedures, such as DIRECT and CMA-ES.

Many other acquisition functions have been proposed, but they often yield similar results; see for example the works of \cite{Mockus:1982} and \cite{Jones:2001}.
The idea of learning portfolios of acquisition functions online was explored by~\cite{Hoffman:2011}.
We do not consider these acquisition functions for brevity.
The BO procedure is summarized in Algorithm~\ref{alg:bo}.

\begin{algorithm}
\caption{GP-UCB}
\label{alg:bo}
\begin{algorithmic}
{ \small
\FOR{$t=1,2,\dots$}
  \STATE $\vx_{t+1} = \argmax_{\vx \in {\cal X}} \calU(\vx|\mathcal{D}_t).$
  \STATE Augment the data $\mathcal{D}_{t+1} =
  \{\mathcal{D}_{t}, (\vx_{t+1}, f(\vx_{t+1})) \}$
\ENDFOR
}
\end{algorithmic}
\end{algorithm}

Finite sample bounds for GP-UCB were derived by~\cite{Srinivas:2010}.
However, the bounds
depend on the algorithm being able to optimize the UCB acquisition function, at each iteration,
exactly. Unless the action set is discrete, it is unlikely that we will be able to find the global optimum of the UCB with a fixed budget optimization method. That is,
we may not be able to guarantee that we can find the exact optimum of the UCB, and hence the theoretical bounds seem to make a very strong assumption in this regard.

\subsection{Shrinking feasible regions}
\label{sec:dFMS}

\begin{figure}[t!]
 \begin{center}
   \includegraphics[scale=0.34]{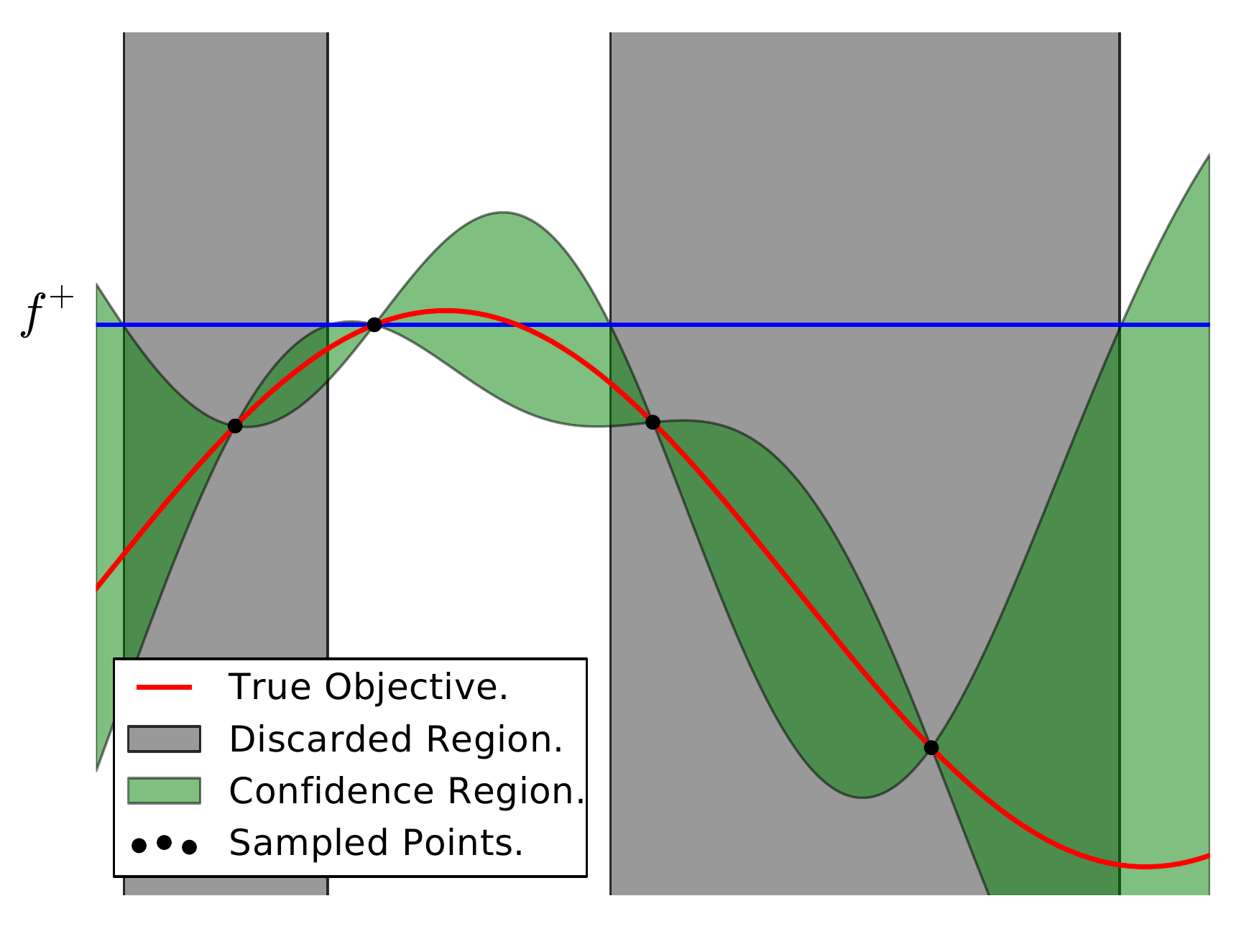}
   \caption{The global shrinking method of \cite{deFreitas:2012}. If the unknown objective function lies within the (green) confidence bounds with high probability, we can discard regions of the space where the upper bound is lower than the best lower bound encountered thus far.}
   \label{feasible}
 \end{center}
\end{figure}

\cite{deFreitas:2012} introduced a different GP-based scheme to
trade off exploration and exploitation. Instead of optimizing the acquisition
function, they proposed to sample the objective function
using a finite lattice within a feasible region $R$. The feasible region at the $t^{th}$ iteration is defined as
$$R_t = \{\vx:\mu_t(\vx) + B_t \sigma_t{(\vx)} > \hspace{-2mm} \sup_{\vx \in R_{t-1}}
\mu_t(\vx) - B_t \sigma_t{(\vx)}\}.$$
That is, one should only search in the region where the upper bound is greater than the best lower bound encountered thus far, as illustrated in~Figure \ref{feasible}. With high probability, the optimizer lies within $R_t$.

\cite{deFreitas:2012} proved that if we double the density of points in the lattice
at each iteration, the feasible region
shrinks very quickly. More precisely, they showed that the simple regret vanishes at an exponential rate and that the cumulative regret is bounded by a constant.

With this approach, they did not have to resort to optimizing an acquisition
function. However, even in moderate dimensions, their algorithm is
impractical since the lattice often becomes too large to be sampled
in a reasonable amount of time.

In this paper, we will argue that to overcome this problem, an optimistic strategy may have to be employed. Such a strategy enables us to
sample the most promising regions first, so as to avoid the computational cost
associated with covering the whole space. In the next section, we begin our discussion of optimistic strategies.

\section{Simultaneous optimistic optimization}
\label{sec:soo}

\begin{algorithm}[t!]
\caption{SOO}
\label{alg:soo}
\begin{algorithmic}
{ \small
\STATE Evaluate $f(\vx_{0, 0})$
\STATE Initialize the tree $\mathcal{T}_1 = \{0, 0\}$
\STATE Set $n=1$
\WHILE{true}
    \STATE Set $\nu_{\max} = -\infty$
    \FOR{$h=0:\min\{\mbox{depth}{(\calT_n}), h_{\max}(n)\}$}
        \STATE Select $(h, j) = \argmax_{j \in \{j| (h, j)\in L_n\}}f(\vx_{h,j})$
        \IF {$f(\vx_{h,j}) > \nu_{\max}$}
            \STATE Evaluate the children of $(h, j)$
            \STATE Add the children of $(h, j)$ to $\calT_n$
            \STATE Set $\nu_{\max} = f(\vx_{h,j})$
            \STATE Set $n= n + 1$
        \ENDIF
    \ENDFOR
\ENDWHILE
}
\end{algorithmic}
\end{algorithm}

Deterministic optimistic optimization (DOO) and simultaneous optimistic optimization (SOO) are tree-based space partitioning methods for black-box function optimization \citep{Munos:2011,Munos:2014}. They were inspired by the UCT algorithm, which enjoyed great success in planning \citep{Kocsis:2006}. UCT was shown to have no finite-time guarantees by \cite{Coquelin:2007}. This prompted the development of a range of optimistic, in the face of uncertainty, approaches. The term optimism, here, refers to the fact that the strategies expand at each round tree cells that may contain the optimum.

DOO and SOO partition the space $\calX$ hierarchically
by building a tree. Let us assume that each node of the tree has $k$ children.
A node $(h,j)$ at level $h$ of the tree has children
$\{ (h+1, kj+i)\}_{0\leq i < k-1}$. The children partition the parent cell
$X_{h, j}$ into cells $\{X_{h+1, kj+i},\; 0\leq i < k-1 \}$. The
root cell is the entire space $\calX$.
A node is always evaluated at the center of the cell, which we denote as $\vx_{h, j}$.

Instead of assuming that
the target function is a sample from a GP, DOO and SOO assume
the existence of a symmetric
semi-metric $\ell$ such that $f(\vx^*) - f(\vx) \leq \ell(\vx, \vx^*)$
where $\vx^*$ is the maximizer of $f$.
Although, SOO assumes that $\ell$ exists, it does not require explicit knowledge
of it.

DOO on the other hand does require knowledge of $\ell$. DOO builds a tree $\calT_n$ incrementally, where $n$ denotes the index over node expansions.
DOO expands a leaf $(h,j)$ from the set of leaves $L_n$ (nodes whose children are not in $\calT_n$)
if it has the the highest upper bound:
$f (\vx_{h,j} ) + \sup_{\vx \in X_{h,j}} \ell (\vx_{h,j} , \vx)$. This value for any cell containing $\vx^*$ upper bounds the best function value $f^*$. The performance of DOO depends crucially on our knowledge of the true local smoothness of $f$. SOO aims to overcome the difficulty of having to know the true local smoothness.

SOO, as summarized in Algorithm~\ref{alg:soo}, expands several leaves simultaneously. When a node is expanded, its children are evaluated. At each round, SOO expands at most one leaf per level, and a leaf is
expanded only if it has the largest value among all
leaves of the same or lower depths.
The SOO algorithm takes as input a function $n \rightarrow h_{\max}(n)$,
which limits the maximum height of the tree after $n$ node expansions.
$h_{\max}(n)$ defines a tradeoff between deep versus broad exploration.
At the end of the finite horizon, SOO returns the $\vx$ with the highest objective function value.
Figure~\ref{fig:tree} illustrates the application of SOO to a simple 1-dimensional optimization problem.


\begin{figure}[t!]
\begin{center}
  \includegraphics[scale=0.36]{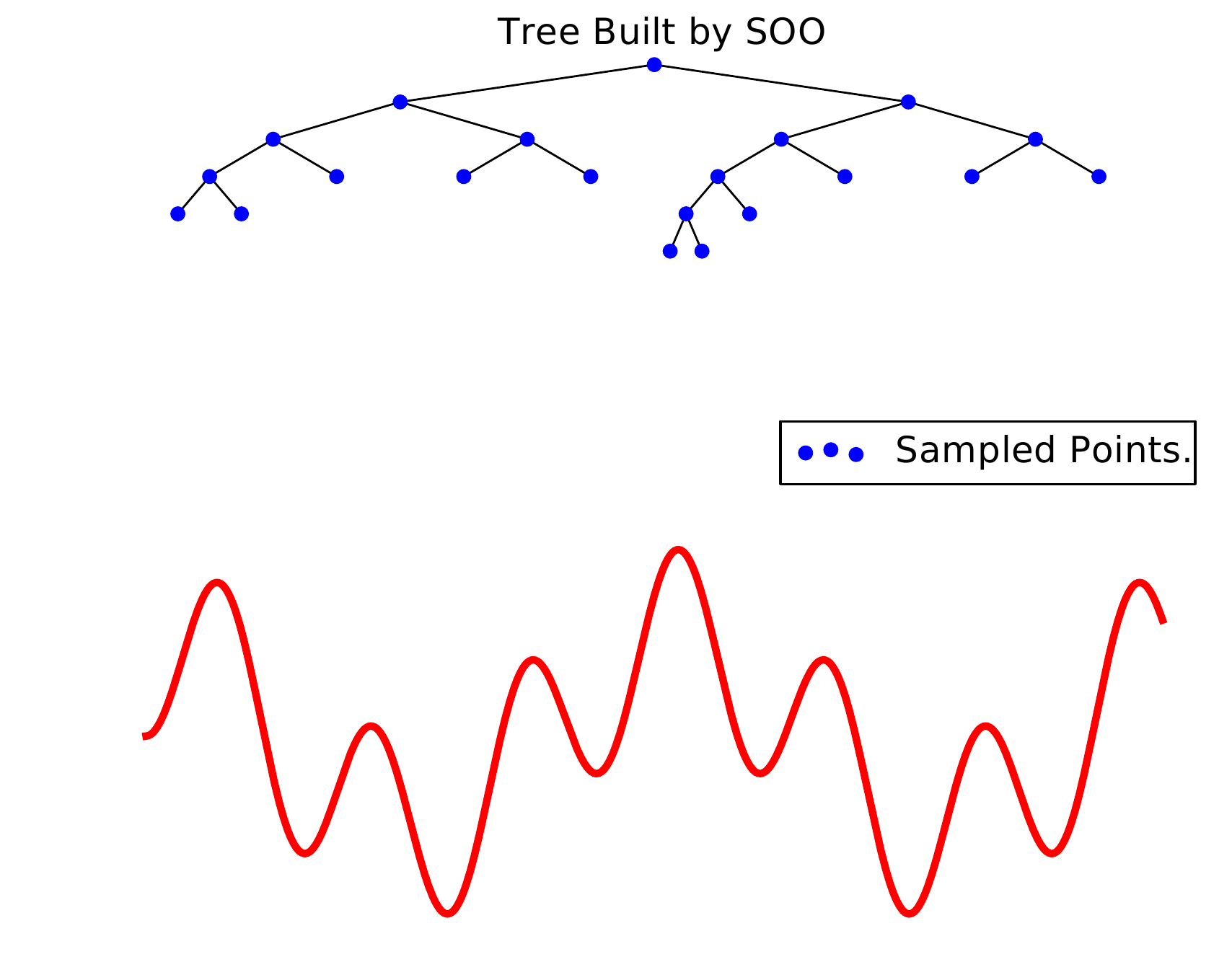}
  \includegraphics[scale=0.36]{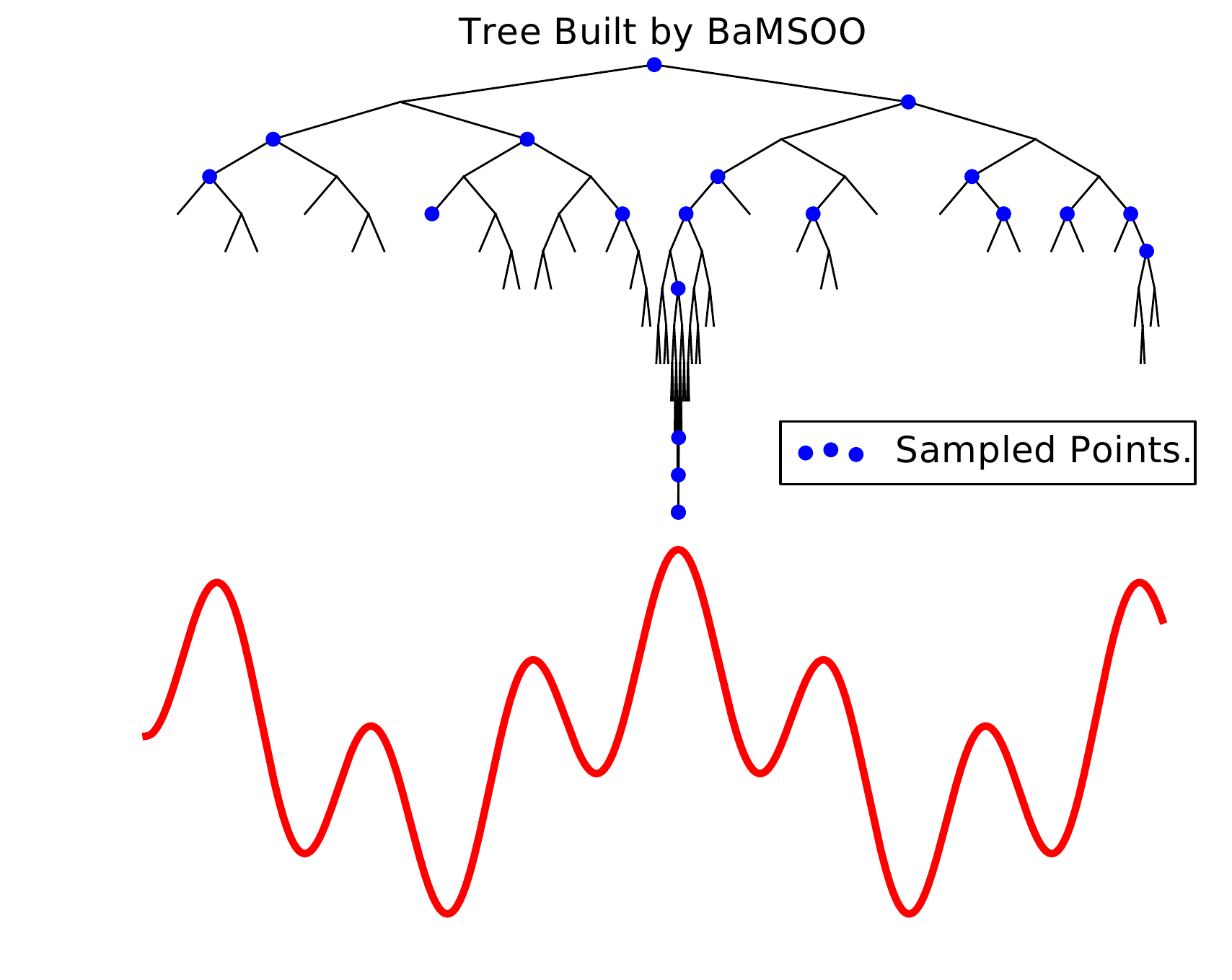}
  \caption{[TOP]: The tree built by SOO when optimizing the function
  $f(x) = \frac{1}{2}\sin(15 x) \sin(27 x)$ in $[0, 1]$.
  [BOTTOM]: The tree built by BaMSOO.
  The $20$ blue dots represent nodes where the objective was evaluated.
   BaMSOO, in comparison, does not
  evaluate the objective function for points known to be sub-optimal with high probability. Hence,
  BaMSOO can achieve a better coverage of the search space with the same number of function evaluations as SOO.}
  \label{fig:tree}
\end{center}
\end{figure}

\section{BaMSOO}
SOO offers a different way of trading off exploration and exploitation 
that does not require the optimization of an acquisition function.
However, it does not
utilize all the information brought in by the previously evaluated
points effectively. To improve upon SOO in practice, we consider the additional assumption
that the objective function is a sample from a GP prior.

We define the LCB and UCB to be
$\calL_N(\vx|\mathcal{D}_{t}) =
 \mu(\vx|\data_t) - B_N \sigma(\vx|\data_t)$
and
$\calU_N(\vx|\mathcal{D}_{t}) =
 \mu(\vx|\data_t) + B_N \sigma(\vx|\data_t)$
where $B_N = \sqrt{2\log(\pi^2N^2/{6\eta})}$ and $\eta \in (0,1)$.

The BaMSOO algorithm is very similar to SOO. As with SOO, we only evaluate the cell at the center point.
However,
when a node's UCB is less than
the function value of the best point already sampled, denoted $f^{+}$,
we do not evaluate the objective function at this node
because with high probability the center point is
 sub-optimal. Instead, we simply assign to this node its LCB value. Note that if the center-point of a cell is sub-optimal, the cell may still contain the optimizer. Hence this cell must also be further expanded in subsequent iterations. To manage these two types of node in the pseudo-code (see Algorithm~\ref{alg:sooucb}), we introduce a place-holder function $g$ which is set to $f$ when the UCB of the cell of interest is bigger than $f^{+}$, and it is set to the LCB of the node otherwise. For clarity, we remind the reader that the indices $N, k, t$ and $n$ are over node evaluations, branches (children), function evaluations and node expansions respectively.

In the pseudocode, we have highlighted in blue the additional lines of code brought in by BaMSOO. Effectively, BamSOO only involves a slight modification of SOO (Algorithm 2) provided we have GP routines to evaluate the LCB and UCB.

We found the assignment of the LCB values to nodes that do worse than $f^{+}$ to work well in practice. For this reason our presentation, experiments and theory focus on this choice.

BaMSOO improves upon SOO by making use of the
available information more efficiently. Moreover, by using an optimistic proposal, it avoids
the need to sample exhaustively before shrinking the feasible region as in \citep{deFreitas:2012}.
Figure~\ref{fig:tree} illustrates how BaMSOO can cover the search space more effectively, even though it incurs the same number of expensive function evaluations as SOO.

\begin{algorithm}[h!]
\caption{BaMSOO}
\label{alg:sooucb}
\begin{algorithmic}[1]
{ \small
\STATE Set $g_{0, 0}=f(\vx_{0, 0})$
\STATE \textcolor{blue}{Set $f^+=g_{0, 0}$}
\STATE Initialize the tree $\mathcal{T}_1 = \{0, 0\}$
\STATE \textcolor{blue}{Set $t=1$,} $n = 1, N=1$\textcolor{blue}{,
 and $\mathcal{D}_{t} = \{(\vx_{0, 0}, g(\vx_{0, 0})) \}$}
\WHILE{true}
    \STATE Set $\nu_{\max} = -\infty$.
    \FOR{$h=0$ to $\min\{\mbox{depth}{(\calT_n}), h_{\max}(n)\}$}
        \STATE Select $(h, j) = \argmax_{j \in \{j| (h, j)\in L_n\}} g(\vx_{h,j})$
        \IF {$g(x_{h,j}) > \nu_{\max}$}
            \FOR{$i = 0$ to $k-1$}
                \STATE Set $N = N + 1$
                \IF {\textcolor{blue}{$\calU_N(\vx_{h+1, kj+i}|\mathcal{D}_{t}) \geq f^{+}$}}
                    \STATE Set $g(\vx_{h+1, kj+i}) = f(\vx_{h+1, kj+i})$
                    \STATE \textcolor{blue}{Set $t = t + 1$}
                    \STATE
                    \textcolor{blue}{$\mathcal{D}_{t} =
                    \{\mathcal{D}_{t-1}, (\vx_{h+1, kj+i}, g(\vx_{h+1, kj+i})) \}$}
                \ELSE
                    \STATE \textcolor{blue}{Set $g(\vx_{h+1, kj+i}) = \calL_N(\vx_{h+1, kj+i}|\mathcal{D}_{t})$}
                \ENDIF

                \IF  {\textcolor{blue}{$g(\vx_{h+1, kj+i}) > f^+$}}
                    \STATE \textcolor{blue}{Set $f^+ = g(\vx_{h+1, kj+i})$}
                \ENDIF
            \ENDFOR

            \STATE Add the children of $(h, j)$ to $\calT_n$
            \STATE Set $\nu_{\max} = g(\vx_{h,j})$
            \STATE Set $n = n + 1$
        \ENDIF
    \ENDFOR
\ENDWHILE
}
\end{algorithmic}
\end{algorithm}

\section{Analysis}

In this section, we provide an overview of the theoretical analysis of BaMSOO, which appears in the Appendix. Our discussion here will focus on our assumptions. At the end of this section, we will present the main result and sketch the proof coarsely.

We denote the global maximum by $f^* = \sup_{\vx \in \calX} f(\vx)$ and
the maximizer by $\vx^* = \argmax_{\vx \in \calX} f(\vx)$.

We make similar assumptions to those made by~\cite{deFreitas:2012}.
 As in their case, we make the global assumption
 that the objective function is a sample from a GP and a local assumption about
 the behavior of the objective near the optimum.

 \begin{assumption}[Conditions on the GP kernel]
 \label{kernelAspn}
 $\calX \subseteq \mathbb{R}^D$ is a compact set, and $\kappa$ is a kernel
 on $\mathbb{R}^D$ that is twice differentiable along the diagonal
 such that
 $\partial_\vx \partial_{\vx'}
      \kappa(\vx,\vx')|_{\vx=\vx'}$ exists.
\end{assumption}

\begin{assumption} [Local smoothness of $f$]
 \label{envelop}
 $f \sim \mbox{GP}(0, \kappa)$ is a continuous sample on $\calX$ that
 has a unique global maximum $\vx^*$, such that
 $f^* - c_1\|\vx - \vx^*\|^{\alpha}_2 \leq f(\vx)$
 $\forall \vx \in \calX$ and
 $f(\vx) \leq f^* - c_2\|\vx - \vx^*\|^2_2$
 $\forall \vx \in \calB(\vx^*, \rho)$ for some constants $c_1$,  $c_2$, $\rho > 0$
 and $\alpha \in \{1, 2\}$.
 Also $f^* - \max_{\vx \in \calX \setminus \calB(\vx^*, \rho)} f(\vx) > \epsilon_0$ for some
 $\epsilon_0 > 0$.
\end{assumption}

 As argued by~\cite{deFreitas:2012}, in many practical cases
 the local conditions follow almost surely from the global condition.
 For example, if we were to employ the Matern kernel with $\nu > 2$ or
 a kernel that is $6$ times differentiable along the diagonal,
 we would have that the samples of the GPs are twice differentiable
 with probability one.
 The first case was shown by~\cite[Theorem 1.4.2]{Adler:2007}
 and \cite[\S2.6]{Stein:1999}, while the second result was shown by~\cite[Theorem 5]{Ghosal:2006}.
 If the $\vx^*$ lies in the interior of $\calX$,
 then the Hessian of $f$ at $\vx^*$ would be almost surely non-singular as
 at least one of the eigenvalues of the Hessian is a co-dimension
 1 condition in the space of all functions
 that are smooth at a given point~\citep{deFreitas:2012}.
 In this case, we would have that
 $$f^* - c_1\|\vx - \vx^*\|^{\alpha}_2 \leq f(\vx)
 \leq f^* - c_2\|\vx - \vx^*\|^2_2$$ with $\alpha = 2$.

If $\vx^*$ lies on the boundary
 of $\calX$ which we assume to be smooth,
 then $\nabla f (\vx^*) \neq 0$
 since the additional event of the vanishing
 of $\nabla f (\vx^*)$ is a co-dimension $d$ phenomenon in the
 space of functions with global maximum at $\vx^*$~\citep{deFreitas:2012}.
 In this case, we would have that
 $$f^* - c_1\|\vx - \vx^*\|^{\alpha}_2 \leq f(\vx)
 \leq f^* - c_2\|\vx - \vx^*\|^2_2$$ with $\alpha = 1$.

 Finally, a sample from a GP on a compact domain has a unique maximum
 with probability one. This is because the space
 of continuous functions on a compact domain that attain
 their global maximum at more than one point have
 co-dimension 1 in the space of
 all continuous functions on that domain~\citep{deFreitas:2012}.

The subsequent assumptions are about the hierarchical partitioning of the search space. They are the same
as Assumptions 3 and 4 in~\cite{Munos:2011}.

\begin{assumption}[Bounded diameters]
 \label{as:bd}
 There exists a decreasing sequence $\delta(h) > 0$, such that for
 any depth $h \geq 0$, for any cell $X_{h,i}$ of depth $h$, we have
 $\sup_{\vx\in X_{h,i}} \ell(\vx_{h,i} , \vx) \leq \delta(h).$
 Here $\ell(\vx, \vy) := c_1 \|\vx- \vy\|^{\alpha}_2$
 where $\alpha \in \{1, 2\}$ and
 $\delta(h) = c\gamma^{h}$ for some constant $c >0$ and $\gamma \in (0, 1)$.
\end{assumption}

\begin{assumption}[Well-shaped cells]
 \label{as:cell}
 There exists $\nu > 0$
 such that for any depth $h \geq 0$, any cell $X_{h,i}$
 contains an $\ell$-ball of radius $\nu\delta(h)$ centered in $X_{h,i}$.
\end{assumption}

Note that depending on the value of $\alpha$, $\gamma$ would have to take on a different
value for Assumptions~\ref{as:bd} and~\ref{as:cell} to be
satisfied. Regardless of the choice of $\alpha$ and as illustrated in Example 1 of~\cite{Bubeck:2011},
Assumptions~\ref{as:bd}
and~\ref{as:cell} are easy to satisfy in practice; for example when $\vx \in [0,1]^D$ and the split is done along the largest dimension of a cell. This is the case in all our experiments.

 Assumption~\ref{envelop} together with Assumptions~\ref{as:bd} and~\ref{as:cell}
 impose a ``\emph{near optimality}'' condition
 as defined by~\cite{Munos:2011}.

We can now present our main result, which is in the form of a corollary to Theorem 1 in the Appendix.

\begin{corollary}
\label{simpleReg}
 Let $d = -(D/4-D/\alpha)$ and $h_{\max}(n) = n^{\epsilon}$.
 Given Assumptions $1-4$, we have that with probability at least $1-\eta$,
 the loss of BaMSOO is $\mathcal{O}\left(n^{-\frac{1-\epsilon}{d}}
 \log^{\frac{\alpha}{4-\alpha}}(n^2/\eta) \right)$.
\end{corollary}

It is worth pointing out that the result presented in Corollary~\ref{simpleReg}
is based on the number of node expansions $n$ instead of the number of function evaluations. The theory can therefore be strengthened.

If $\alpha=2$ and $\epsilon=1/2$, then the above result translates to
$\mathcal{O}\left(n^{-\frac{2}{D}}\log(n^2/\eta) \right)$.
If $\alpha=1$ with $\epsilon$ being the same as before, then
the rate of convergence becomes
$\mathcal{O}\left(n^{-\frac{2}{3D}}\log^{\frac{1}{3}}(n^2/\eta) \right)$.

 The structure of the proof follows that in~\cite{Munos:2011}.
 Let $\vx^*_h$ denote the optimal node at level $h$ (that is, the node at height $h$ in the branch that contains the optimum $\vx^*$).
 Our proof shows that once $\vx^*_h$ is expanded, it does not
 take long for $\vx^*_{h+1}$ to be expanded.
 Once an optimal node $\vx^*_h$ is expanded,
 by Assumptions~\ref{envelop} and~\ref{as:bd}, we have that
 the loss of BaMSOO is no worse than $\delta(h)$ .

 The main difficulty of the proof lies in the fact that we sometimes do not sample
 nodes when their UCB values are less than the best observed value.
 In this case, we can no longer make the claim that an optimal node
 is expanded soon after its parent.
 This is because when
 a node is not expanded, its LCB can be very low due to a high standard deviation.
 Fortunately, we can show that this is not the case for optimal nodes
 in the optimal region.
 This is accomplished by showing that the standard deviation at
 a point is no more than its distance to the nearest sampled point
 up to a constant factor (shown in Lemma~\ref{lem:varbound}).
 This enables us to show that every optimal node in the optimal region
 must have a low standard deviation.
 Given this result, we can adopt the proof structure outlined in
 \cite{Munos:2011}.

\section{Experiments with global optimization benchmarks}

\begin{figure*}[t!]
\begin{center}
  \includegraphics[scale=0.3]{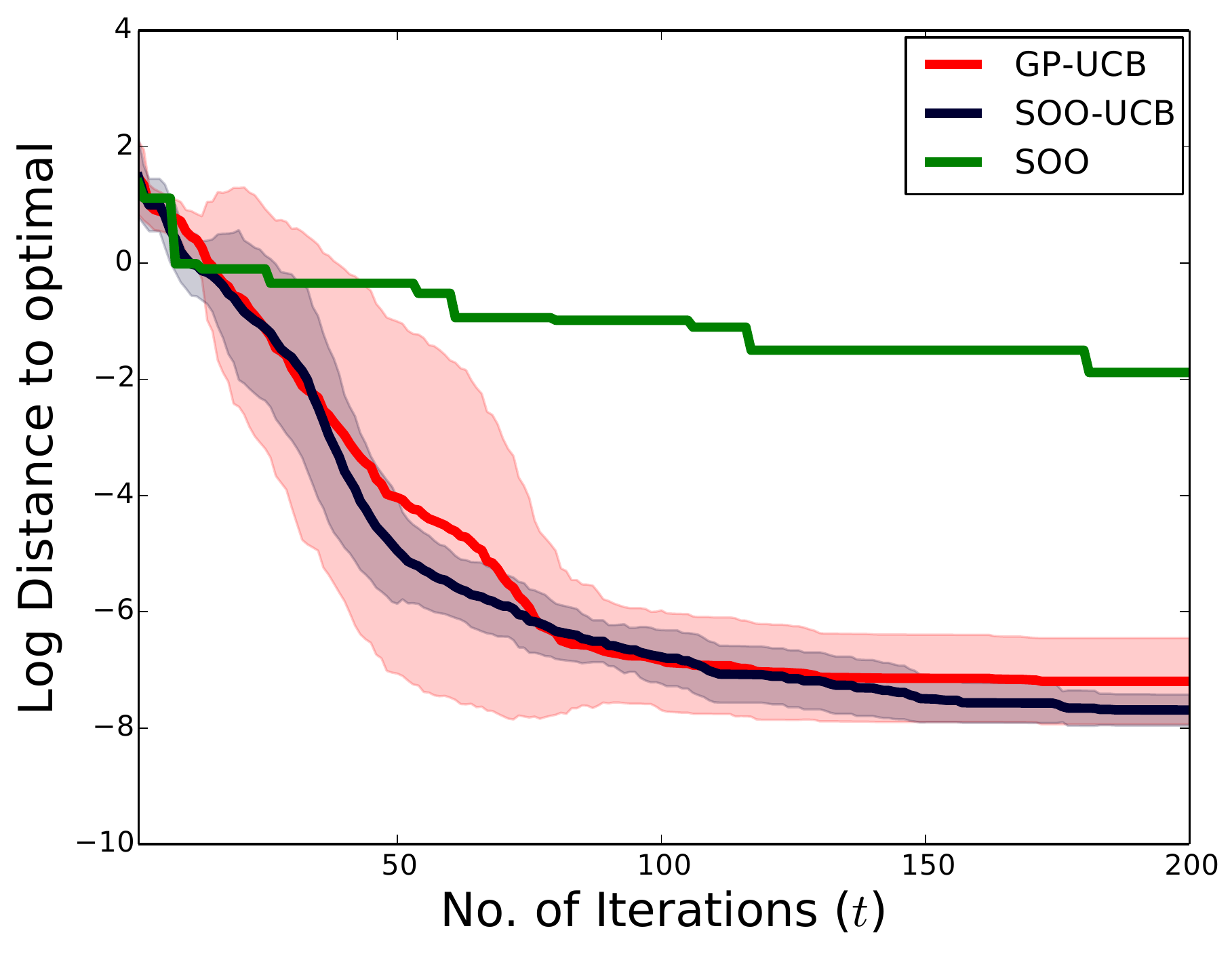}
  \includegraphics[scale=0.3]{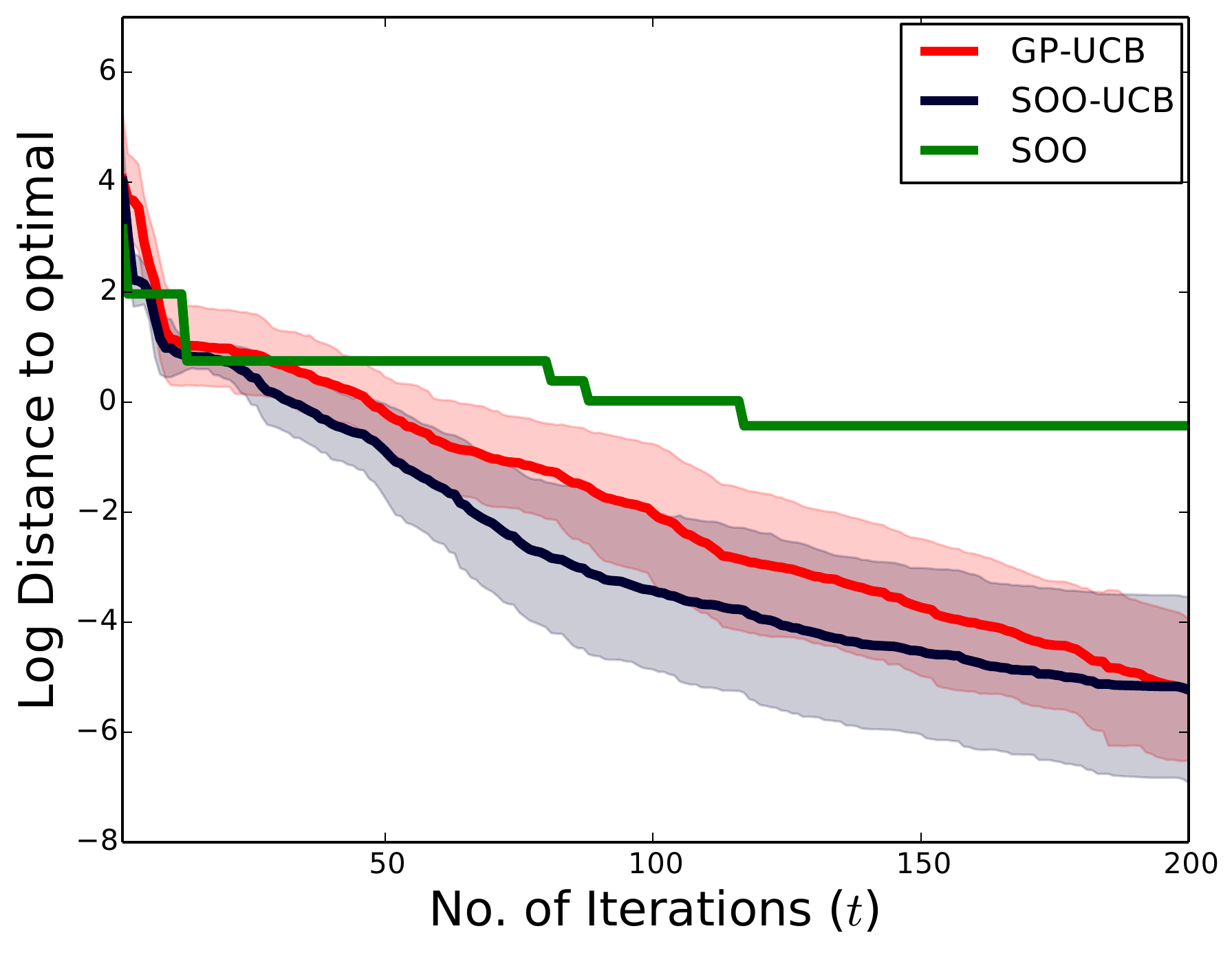}
  \includegraphics[scale=0.3]{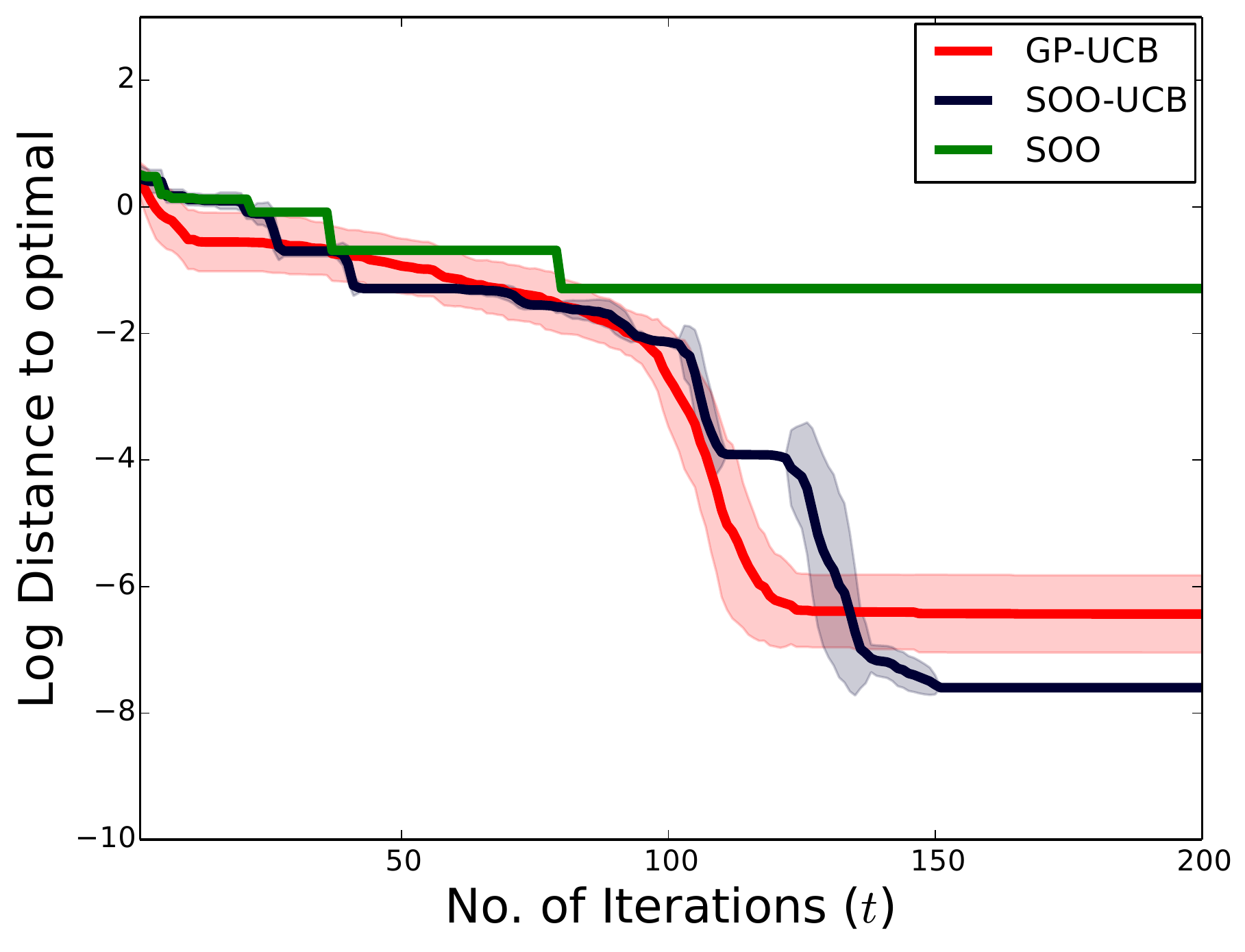}
  \caption{\label{fig:lo}Comparison of GP-UCB, SOO, and BaMSOO on multi-modal test functions
  of low dimensionality (Branin, Rosenbrock and Hartmann3D). GP-UCB and BaMSOO
  perform similarly whereas SOO does poorly. The poor performance of SOO is caused
  by having weaker assumptions on the smoothness of the objective function.
  The good performance of GP-UCB indicates that when the dimensionality is low
  optimizing the acquisition function is reasonable.}
\end{center}
\end{figure*}

In this section, we validate the proposed algorithm with a series of experiments
that compare the three algorithms (GP-UCB, SOO, BaMSOO) on global optimization benchmarks.
We have omitted the feasible region shrinking algorithm
(described in Section~\ref{sec:dFMS}) as it is not practical
for problems of even moderate dimensions. We have also omitted comparisons to PI and EI as these appear in \cite{Hoffman:2011} for the optimization benchmarks described in this paper.

In our experiments, we used the same hyper-parameters in
 GP-UCB and BaMSOO for each test function. We also randomized
the initial sample point for BaMSOO and GP-UCB so that they are not deterministic.
To optimize the acquisition function for GP-UCB, we used DIRECT followed by
a local optimization method using gradients.

We use 5 test functions: Branin, Rosenbrock, Hartmann3, Hartmann6,
and Shekel.
All of these test functions are common in the global optimization literature and with
the exception of the Rosenbrock, they are all multi-modal.
\footnote{Detailed information about the test functions is available at the following website:
{\tiny \url{http://www-optima.amp.i.kyoto-u.ac.jp/member/student/hedar/Hedar_files/TestGO_files/Page364.htm}.}}

We rescaled the domain of each function to the $[0, 1]^D$ hypercube, and
we used the log distance to the true optimum as our evaluation metric. This metric is
defined as $\log_{10}(f^* - f^+)$ where $f^+$ is the best objective value
sampled so far and $f^*$ is the true maximum value of the objective.
For each test function, we repeat our experiments $50$ times for GP-UCB
and BaMSOO and run SOO once as SOO is a deterministic strategy. We plot the
mean and a confidence bound of one standard deviation of
our metric across all the runs for all the tests.

\begin{figure}[b!]
\begin{center}
  \includegraphics[scale=0.3]{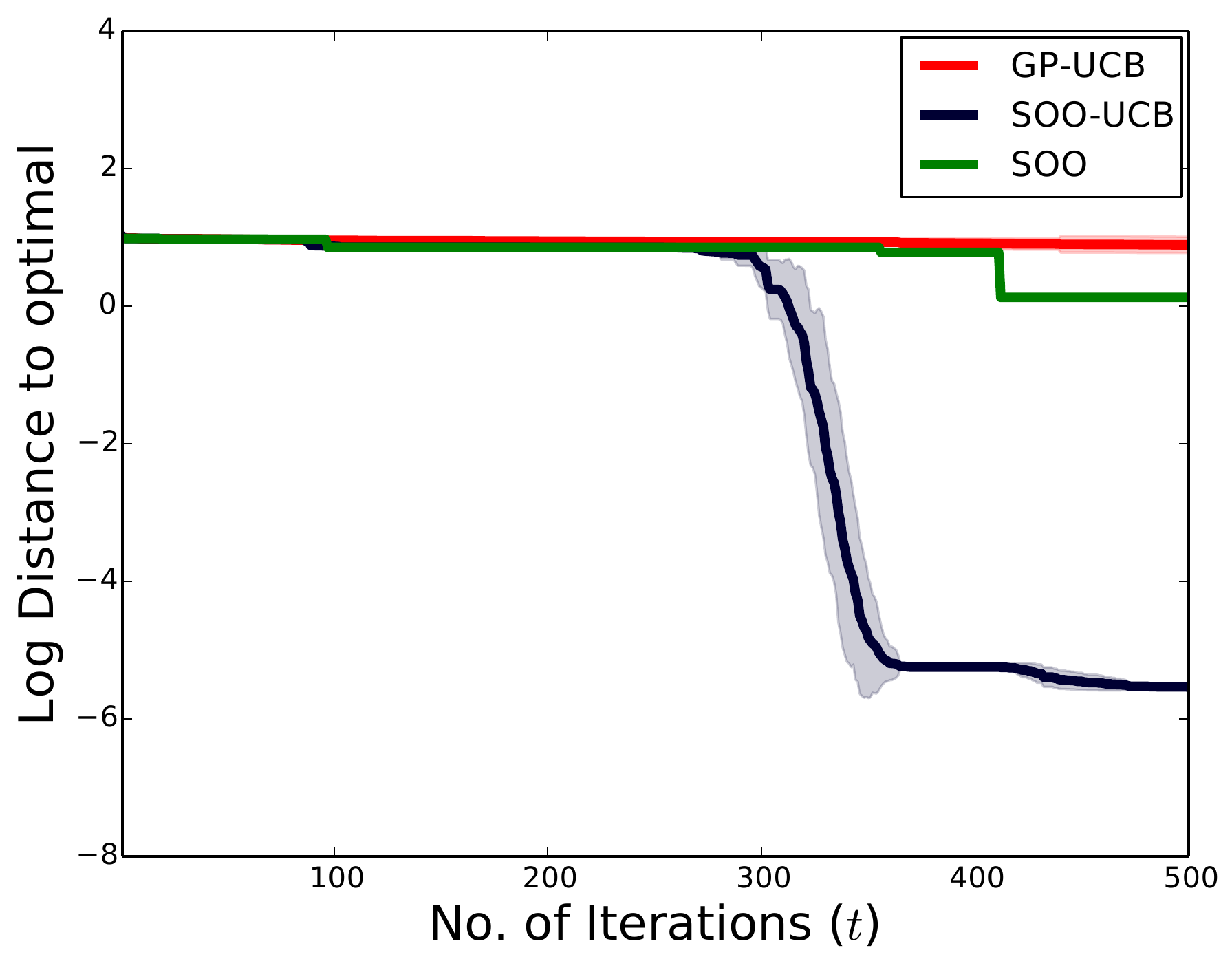}
  \includegraphics[scale=0.3]{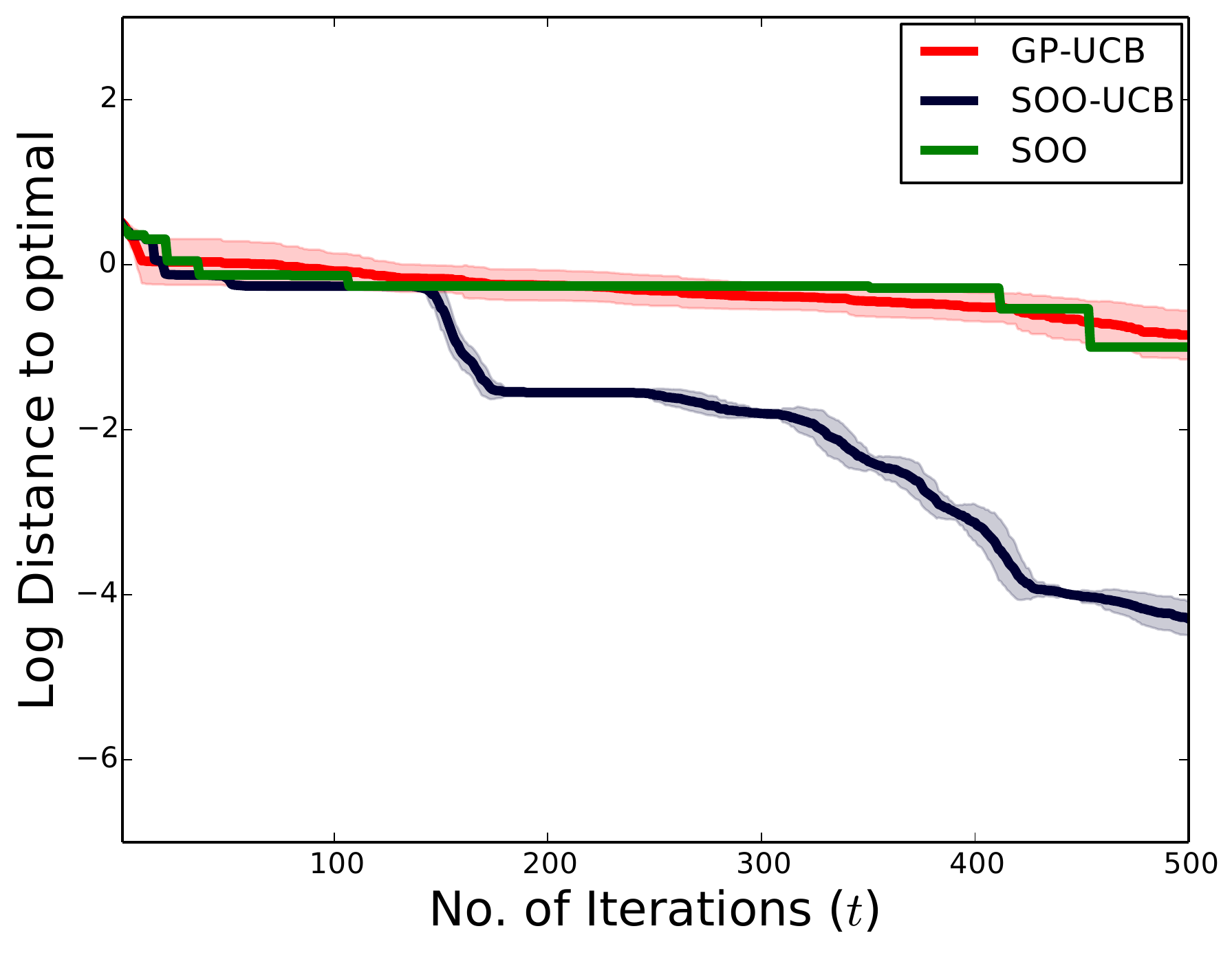}
  \caption{\label{fig:higherd}Comparison of GP-UCB, SOO, and BaMSOO on multi-modal test functions
  of moderate dimensionality: 4D Shekel function (top) and 6D Hartmann function (bottom).
  Here, GP-UCB performs poorly.
  This is due in part to the hardness of optimizing the acquisition function.}
\end{center}
\end{figure}

\begin{table*}
\label{tbl:time}
\caption{Time required for the test functions measured in seconds.
SOO is very fast as it does not
maintain a GP. BaMSOO maintains a GP to produce
more accurate posterior estimates and is hence slower.
The rejection of proposals also results in bigger trees,
further slowing down the algorithm.
GP-UCB is slow compared to the other two algorithms as it
not only maintains a GP but also
optimizes its acquisition function at each iteration.
}
 \begin{center}
 {\small
  \begin{tabular}{ l || r  r  r  r  r }
    \thickhline
    Algorithm  & Branin &  Rosenbrock &  Hartmann3 & Hartmann6 & Shekel \\
    \thickhline
    GP-UCB  & 29.9438 &  29.5716 &  34.0311 & 115.2402 & 100.7770 \\
    BaMSOO & 3.0680  &  3.4693  &  3.9722  &  2.0918  &  3.8951 \\
    SOO & 0.1810 &  0.1835 & 0.1871 & 0.4313 &  0.4350 \\
    \thickhline
  \end{tabular}
  }
\end{center}
\end{table*}

For simplicity, we only consider binary-trees for space partitioning
in SOO and BaMSOO. Specifically, the largest dimension in the parent's cell is split to create two children.

First, we test the global optimization schemes on $3$ test
functions with low dimensionality: Branin, Rosenbrock and Hartmann3.
The Branin function~\citep{Jones:2001} is
a common benchmark for Bayesian optimization
and has 2 dimensions. The Rosenbrock function
is a commonly used non-convex test function for local optimization algorithms,
and although it is unimodal, its optimum lies in a long narrow valley, which makes
the function hard to optimize.
Finally, the Hartmann3 function is $3$-dimensional and has four local optima.

As we can see from Figure~\ref{fig:lo}, BaMSOO performs competitively against
GP-UCB on these low dimensional test functions.
Both BaMSOO and GP-UCB achieve very high accuracies of up to $10^{-8}$
in terms of the distance to the optimal objective value.
In comparison, SOO, due to the lack of a strong prior assumption, cannot
take advantage of the points sampled and thus is lagging behind. 

In the experiments shown in Figure~\ref{fig:higherd}, we compare the approaches in
consideration on the Shekel function and the Hartmann6 function.
The Shekel function is 4-dimensional and has 10 local optima.
The Hartmann6 function is 6-dimensional, as the name suggests, and has 6 local
optima.
On these higher dimensional problems, the
performance of GP-UCB begins to dwindle.
Despite the increase in dimensionality,
BaMSOO is still able to optimize the test functions to a relatively
high precision. SOO does not perform as well as BaMSOO again because
of its weak assumptions.
The poor performance of GP-UCB on these two test functions may be due
in part to the inability of a global optimizer to optimize the acquisition
function exactly in each iteration. As the dimensionality increases, so is
the difficulty of optimizing a non-convex function globally as the cost
of covering the space grows exponentially. The optimization of the acquisition function through algorithms like
DIRECT demands the repartitioning of the space in each iteration.
To reach a finer
granularity, we either have to sacrifice speed by building
very fine partitions in each iteration or
accuracy by using coarser partitions.

The proposed approach is not only
competitive with GP-UCB in terms of effectiveness,
it is also more computationally efficient. As we can see in Table~1,
BaMSOO is about 10-40 times faster than GP-UCB on the test functions that we have
experimented with. This is because instead of optimizing the acquisition function
in each iteration the SOO algorithm, that sits inside, only optimizes once.
BaMSOO, however, is much slower than SOO. This is because BaMSOO also employs
a GP to reject points proposed by SOO. To sample one point, SOO may have
to propose many points before one is accepted. For this reason, BaMSOO would build
much bigger trees compared to SOO and it is therefore slower.

\section{Application to term extraction}
In this section, we evaluate the performance of the BaMSOO algorithm on optimizing the parameters in a term extraction algorithm. Term extraction is the process of analyzing a text corpus to find terms, where terms correspond to cohesive sequences of words describing entities of interest. Term extraction tools are widely used in industrial text mining and play a fundamental role in the construction of knowledge graphs and semantic search products. Recently \cite{parameswaran2010towards} proposed a term extraction method, and showed that it outperforms state-of-the-art competitors, but their method has many free parameters that require manual adjustment. Here, we compare the performance of BaMSOO, GP-UCB and SOO in automatically tuning the 4 primary free parameters of the algorithm (support-thresholds).
We define our deterministic objective function to be the F-score of the extracted terms, which is a weighted average of precision and recall. Precision is calculated using a predefined set of correct terms and recall is estimated by simply normalizing the number of extracted correct terms to be in the range [0,1]. We run the experiment on the GENIA corpus \citep{genia}, which is a collection of 2000 abstracts from biomedical articles. The results of the experiment are shown in Figure~\ref{fig:termextraction}. It is evident from this figure that BaMSOO outperforms GP-UCB and SOO in this application.

\begin{figure}[t!]
\begin{center}
  \includegraphics[scale=0.36]{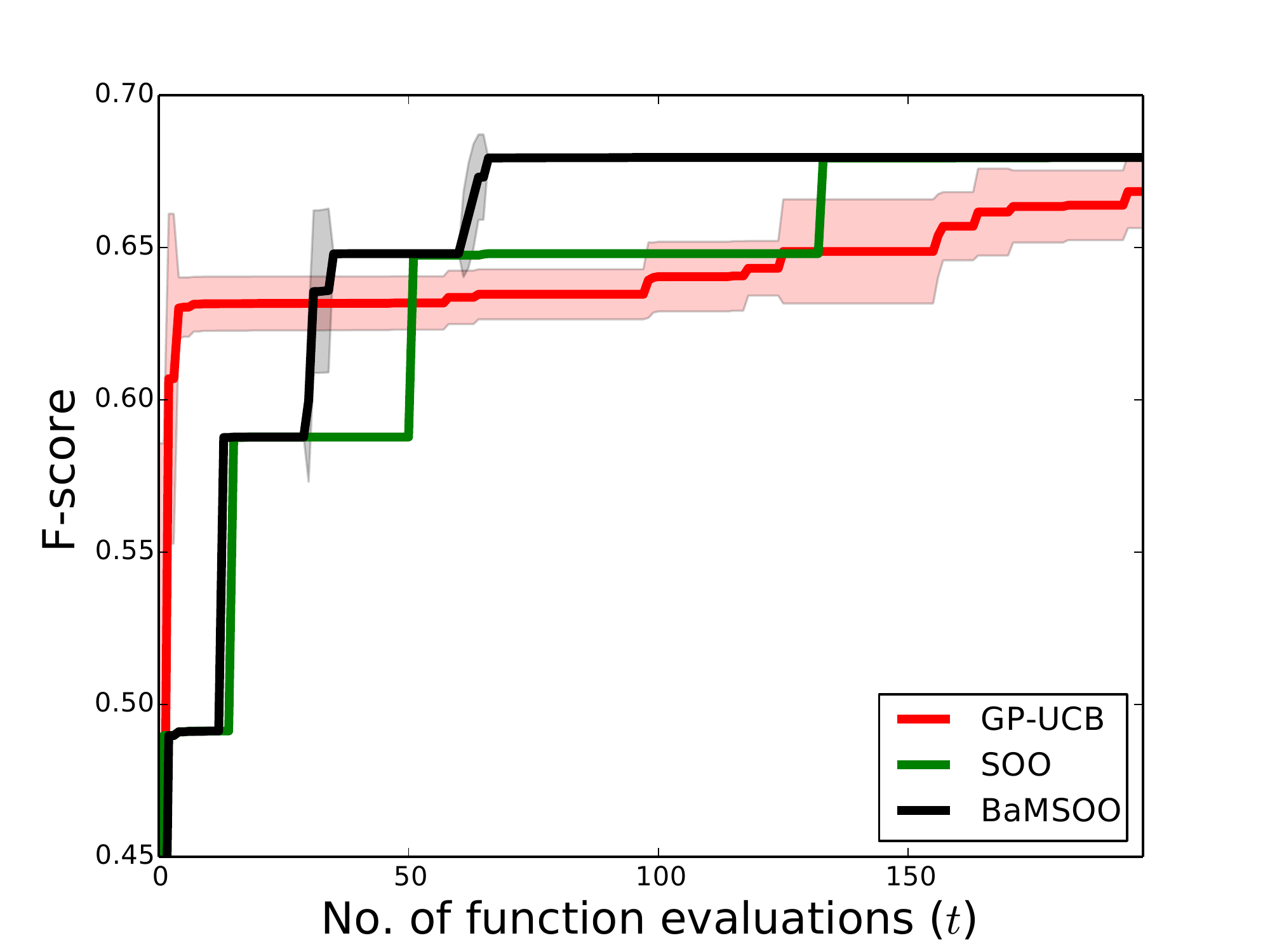}
  \caption{\label{fig:termextraction}Comparison of GP-UCB, SOO, and BaMSOO on optimizing 4 parameters in term extraction from the GENIA corpus using a term extraction algorithm by \citep{parameswaran2010towards}. In this plot, higher is better.}
\end{center}
\end{figure}

\section{Discussion}
 This paper introduced a new global optimization algorithm BaMSOO, which does not require the auxiliary optimization of either acquisition functions or samples from the GP. In trials with benchmark functions from the global optimization literature, the new algorithm outperforms standard BO with GPs and SOO, while being computationally efficient.
The paper also provided a theoretical analysis proving that
the loss of BaMSOO decreases polynomially.


The careful reader may have noticed that, despite the effectiveness of
BaMSOO in the experiments, the convergence rate of BaMSOO is
not as good as that of SOO for $\alpha=2$.
This is because we were only able to prove
that the standard deviation
at a point decreases linearly, instead of quadratically, when a nearby
point is sampled (Lemma~\ref{lem:varbound} in the Appendix).
Since by assumption the objective function behaves quadratically
in the optimal region, the linear decrease of the standard deviation
gives rise to a sub-optimal convergence rate.
It is also interesting to note that the same type of bound
on the standard deviation was used by~\cite{Bull:2011}, who
achieved similar convergence rates to the ones in this paper.
\cite{deFreitas:2012} showed that if the samples
form a $\delta$-cover on a subset of $\calD \subseteq \calX$,
then the standard deviation of all points on $\calD$ is bounded by a quadratic term
$\frac{Q}{4}\delta^2$. Via this observation, the authors achieved a
geometric convergence rate.
The requirement of the $\delta$-cover, however, renders their
algorithm impractical.
Finding a \emph{practical} GP-based algorithm that achieves geometric convergence
rates remains an open problem.

%

\section*{Acknowledgements}
We would like to thank Remi Munos for many valuable discussions.
We also thank NSERC and the University of Oxford for financial support.

{
\bibliography{bayesopt}
\bibliographystyle{icml2014}
}

\appendix
\newpage
\onecolumn
\section{Proofs}

We begin by introducing some notation.
Let $\vx^*_h$ denote the optimal node at level $h$.
That is the cell of $\vx^*_h$ contains the optimizer $\vx^*$.
Also let $f^+$ and $\vx^+$ represent the best function value observed thus far and
the associated node respectively.

\subsection{Technical Lemmas}
\begin{lemma}[Lemma 5 of~\cite{deFreitas:2012}]
  \label{lem:rkhs}
  Given a set of points $\vx_{1:T} := \{\vx_1, \ldots, \vx_T\} \in \mathcal D$ and a
  Reproducing Kernel Hilbert Space (RKHS) $\calH$ with kernel $\kappa$ the following
  bounds hold:
  \begin{enumerate}
  \item \label{item:lipschitz}
    Any $f \in \calH$ is Lipschitz continuous with constant
    $\|f\|_\calH L$, where $\|\cdot\|_\calH$ is the Hilbert space norm
    and $L$ satisfies the following:
    \begin{align*}
      L^2 \leq \sup_{\vx \in \mathcal D} \partial_\vx \partial_{\vx'}
      \kappa(x,x')|_{\vx=\vx'}
    \end{align*}
    and for $\kappa(\vx,\vx') = \widetilde{\kappa}(\vx-\vx')$ we have
    \[  L^2 \leq \partial_\vx^2 \widetilde{\kappa}(\vx)|_{x=0}. \]
  \item \label{item:projection}
    The projection operator $P_{1 : T}$ on the subspace $\displaystyle \Span_{t = 1: T} \{\kappa(x_t,\cdot) \} \subseteq \calH$ is given by
    \begin{align*}
      P_{1:T}f := \vk^\top (\cdot) \vK^{-1} \left< \vk(\cdot), f \right>
    \end{align*}
     where $\vk(\cdot) = \vk_{1:T}(\cdot) := \left[\kappa(\vx_1,\cdot)
     \cdots \kappa(\vx_T,\cdot) \right]^\top$
     and $\vK := \left[ \kappa(\vx_i, \vx_j) \right]_{i,j = 1:T}$; moreover, we have that
     \[ \left< \vk(\cdot), f \right> := \begin{bmatrix} \left< \kappa(\vx_1, \cdot), f \right> \\
     \vdots \\ \left< \kappa(\vx_T, \cdot), f \right> \end{bmatrix} = \begin{bmatrix} f(\vx_1) \\
     \vdots \\ f(\vx_T) \end{bmatrix}. \]
    Here $P_{1:T} P_{1:T} = P_{1:T}$ and $\|P_{1:T}\| \leq 1$ and $\|\indicator - P_{1:T}\|
    \leq 1$.
  \item \label{item:interpolation}
    Given tuples $(\vx_i, f_i)$ with $f_i = f(\vx_i)$, the minimum norm
    interpolation $\bar{f}$ with $\bar{f}(\vx_i) = f(\vx_i)$ is given by
    $\bar{f} = P_{1:T} f$. Consequently its residual $g := ( \mathbf{1} - P_{1:T}) f$
    satisfies $g(\vx_i) = 0$ for all $\vx_i \in \vx_{1:T}$.
  \end{enumerate}
\end{lemma}

\begin{lemma}[Lemma 6 of~\cite{deFreitas:2012}]
 \label{lem:gpvar}
 Under the assumptions of Lemma~\ref{lem:rkhs} it follows that
 $$|f(\vx) - P_{1:T} f(\vx)| \leq \|f\|_\calH \sigma_T(\vx),$$
 where $\sigma_T^2(\vx) = \kappa(\vx,\vx) - \vk_{1:T}^\top(\vx) \vK^{-1} \vk_{1:T}(\vx)$
 and this bound is tight. Moreover, $\sigma_T^2(\vx)$ is the posterior predictive variance
 of a Gaussian process with the same kernel.
\end{lemma}

\begin{lemma}[Adapted from Proposition 1 of~\cite{deFreitas:2012}]
\label{lem:varbound}
 Let $\kappa: \mathbb R^D \times \mathbb R^D \to \mathbb R$ be a kernel that is
 twice differentiable along the diagonal $\{(\vx,\vx) \,|\, \vx \in \mathbb R^D\}$,
 with $L$ defined as in Lemma \ref{lem:rkhs}.\ref{item:lipschitz},
 and $f$ be an element of the RKHS with kernel $\kappa$.
 If $f$ is evaluated at point $\vx$, then for any other point $\vy$
 we have $\sigma_T(\vy) \leq L\|\vx-\vy\|$.
\end{lemma}
\begin{proof}
Let $\calH$ be the RKHS corresponding to $\kappa$ and $f \in \calH$ an arbitrary
element with $g := (\indicator - P_{1:T})f$; the residual defined in
lemma~\ref{lem:rkhs}.\ref{item:interpolation}.
Since $g \in \calH$, we have by Lemma~\ref{lem:rkhs}.\ref{item:lipschitz},
$g$ is Lipschitz.
Thus we have that for any point $\vy$:
\begin{equation}
\label{eqn:lipbound}
|g(\vy)| \leq L \|g\|_{\calH} \|\vy-\vx\| \leq L \|f\|_{\calH} \|\vy-\vx\|,
\end{equation}
where the second inequality is guaranteed
by Lemma~\ref{lem:rkhs}.\ref{item:projection}.
On the other hand, by Lemma~\ref{lem:gpvar},
we know that for all $\vy$ we have the following tight bound:
\begin{equation}
\label{eqn:sigbound}
 |g(\vy)| \leq  \|f\|_{\calH}\sigma_T(\vy)
\end{equation}
Now, given the fact that both inequalities (\ref{eqn:lipbound})
and~(\ref{eqn:sigbound}) are
bounding the same quantity and that the latter
is a tight estimate, we necessarily have that:
$$
\|f\|_{\calH}\sigma_T(\vy) \leq L \|f\|_{\calH} \|\vy-\vx\|.
$$
Canceling $\|f\|_{\calH}$ gives us the result.
\end{proof}


\begin{lemma}[Adapted from Lemma 5.1 of~\cite{Srinivas:2010}]
 \label{sgpbound}
 Let $f$ be a sample from a GP.
 Consider $\eta \in (0, 1)$ and set $B_T = 2\log(\pi_T/\eta)$
 where $\sum_{i=1}^{\infty}\pi_T^{-1} = 1$, $\pi_T > 0$. Then,
 $$|f(\vx_T) - \mu_T(\vx_T)| \leq B_T^{\frac{1}{2}}\sigma_T(\vx_T)
 \mbox{ }\forall T\geq 1$$
 holds with probability at least $1-\eta$.
\end{lemma}
\begin{proof}
 For $\vx_T$ we have that $f(\vx) \sim \calN(\mu_T(\vx_T), \sigma_T(\vx_T))$
 since $f$ is a sample from the GP.
 Now, if $r \sim \calN(0, 1)$, then
 \begin{eqnarray*}
  \mathbb{P}(r>c) &=& e^{-c^2/2}(2\pi)^{-1/2}\int e^{-(r-c)^2/2-c(r-c)} dr \\
  &<& e^{-c^2/2}\mathbb{P}(r>0) = \frac{1}{2}e^{-c^2/2}.
 \end{eqnarray*}
 Thus we have that
 \begin{eqnarray*}
  \mathbb{P}\left(f(\vx) - \mu_T(\vx) > B_T^{1/2} \sigma_T(\vx)\right) =
  \mathbb{P}(r > B_T^{1/2})
  < \frac{1}{2}e^{-B_T/2}.
 \end{eqnarray*}
 By symmetry and the union bound, we have that
 $\mathbb{P}\left(|f(\vx) - \mu_T(\vx)|> B_T^{1/2} \sigma_T(\vx)\right) < e^{-B_T/2}.$
 By applying the union bound again, we derive
 $$\mathbb{P}\left(|f(\vx) - \mu_T(\vx)|> B_T^{1/2} \sigma_T(\vx) \mbox{ }
 \forall T\geq 1\right)
 < \sum_{T=1}^{\infty} e^{-B_T/2}.$$
By substituting $B_T = 2\log(\pi_T/\eta)$, we obtain the result.
 As in~\cite{Srinivas:2010}, we can set $\pi_T = \pi^2T^2/6$.
\end{proof}

Since each node's UCB and LCB are only evaluated at most once, we
give the following shorthands in notation.
Let $N(\vx)$ be the number of evaluations of confidence bounds by the time
the UCB of $\vx$ is evaluated (line 12 of Algorithm~\ref{alg:sooucb})
and let $T(\vx) = |\data_t|$
be the time the UCB of $\vx$ is evaluated.
 Define $\mathcal{U}(\vx) = \calU_{N(\vx)}(\vx|\mathcal{D}_{T(\vx)}) =
 \mu(\vx|\data_{T(\vx)}) + B_{N(\vx)} \sigma(\vx|\data_{T(\vx)})$
and $\mathcal{L}(\vx) = \calL_{N(\vx)}(\vx|\mathcal{D}_{T(\vx)}) =
 \mu(\vx|\data_{T(\vx)}) - B_{N(\vx)} \sigma(\vx|\data_{T(\vx)})$.

\begin{lemma}
 \label{3dopt} Consider $\calB(\vx^*, \rho)$ and $\gamma in (0,1)$ as in Assumptions 2 and 3.
 Suppose $\calL(\vx^*_h) \leq f(\vx^*_h)
 \leq \mathcal{U}(\vx^*_h)$.
 If $\vx_h^* \in \calB(\vx^*, \rho)$ and $\delta(h) <  \epsilon_0$ then
 there exists a constant $\bar{c}$ such that
 $\calL(\vx_h^*) \geq f^* - \bar{c}B_{{N(\vx^*_h)}}\gamma^{\frac{h}{2}}$.
\end{lemma}
\begin{proof}
If $\vx_h^*$ is not evaluated then
$f(\vx^+) \geq \mathcal{U}_T(\vx_h^*) \geq f^* - \delta(h) \geq f^* - \epsilon_0$
which implies that $\vx^+ \in \calB(\vx^*, \rho)$.
Therefore,
$f^* - c_2\|\vx^+ - \vx^*\|^2 \geq f(\vx^+)
\geq \mathcal{U}_T(\vx_h^*) \geq f^* - \delta(h)$
which in turn implies that
$\|\vx^+ - \vx^*\| \leq \sqrt{\frac{\delta(h)}{c_2}}$.
Similarly
$f^* - c_2\|\vx_h^* - \vx^*\|^2 \geq f(\vx_h^*) \geq f^* - \delta(h)$.
Therefore $\|\vx_h^* - \vx^*\| \leq \sqrt{\frac{\delta(h)}{c_2}}$.
By the triangle inequality, we have
$$\|\vx^+ - \vx_h^*\|
\leq \|\vx^+ - \vx^*\| + \|\vx_h^* - \vx^*\|
\leq 2\sqrt{\frac{\delta(h)}{c_2}}.$$
By Lemma \ref{lem:varbound}, we have that
$\sigma_{T(\vx^*_h)}(\vx^*_h) \leq 2L\sqrt{\frac{\delta(h)}{c_2}}$.
By the definition of $\calL_T$, we can argue that
\begin{eqnarray*}
\calL (\vx^*_h) &\geq& \mathcal{U} (\vx_h^*) -
4B_{{N(\vx^*_h)}} L\sqrt{\frac{\delta(h)}{c_2}} \\
&\geq& f^* - \delta(h) - 4B_{{N(\vx^*_h)}} L\sqrt{\frac{\delta(h)}{c_2}} \\
 &=& f^* - c\gamma^h - 4B_{{N(\vx^*_h)}} L\sqrt{\frac{c\gamma^h}{c_2}}.
\end{eqnarray*}
Note that since $\gamma \in (0, 1)$, $\gamma < \gamma^{1/2}$.
Assume that $B_1 = b$.
Let $\bar{c} = c/b+4L\sqrt{\frac{c}{c_2}}$.
Since $B_N > B_1$ $\forall N > 1$, we have the statement.

If $\vx_h^*$ is evaluated then the statement is trivially true.
\end{proof}

\begin{mydefinition}
 Let
 $\bar{\gamma} := \gamma^{\frac{1}{2}}$,
 $\bar{\delta}_{h} := \bar{c} B_{N(\vx^*_h)} \bar{\gamma}^h$, and
 $I_{h}^{\epsilon} = \{ (h, i) : f(\vx_{h, i}) + \epsilon \geq f^* \}$.
\end{mydefinition}




\begin{lemma}
\label{sgpexp}
 Assume that
 $h_{\max} = n^{\epsilon}$.
 For a node $\vx_{h, i}$ at level $h$,
 $B_{N(\vx_{h, i})} = \mathcal{O}(\sqrt{h})$.
\end{lemma}
\begin{proof}
 Assume that there are $n_i$ nodes expanded at the end of iteration $i$
 of the outer loop (the while loop).
 In the $i+1^{th}$ iteration of the outer loop,
 there can be at most $h_{\max}(n_i)$ additional expansions added.
 Thus the total number of expansions at the end of iteration $i$
 is at most $n_{i-1} + h_{\max}(n_{i-1})$.
 We can prove by induction
 that $n_i \leq i^{\frac{1}{1-\epsilon}}$.
 Since any node at level $h$ would be expanded after
 at most $2^h$ iterations, at the time of expansion of any node at level $h$,
 we have that $n < (2^h)^{\frac{1}{1-\epsilon}} = 2^{\frac{h}{1-\epsilon}}$
 where $n$ is the total number of expansions.
 Thus, there would be at most $2\times 2^{\frac{h}{1-\epsilon}}$ evaluations.
 Hence,
 $$B_{N(\vx_{h, i})} \leq \sqrt{2\log(\pi^2 2^{\frac{2h}{1-\epsilon}+2}/{6\eta})}
 \leq \sqrt{2\log(2^{\frac{2h}{1-\epsilon}+2}) + 2\log(\pi^2/{6\eta)}}
 = \mathcal{O}(\sqrt{h}).
 $$
\end{proof}

\begin{lemma}
 \label{fintime}
 After a finite number of node expansions,
 an optimal node $\vx_{h_0}^* \in \calB(\vx^*, \rho)$ is expanded
 such that
 $\bar{c}B_{N(\vx_{h_0}^*)}\bar{\gamma}^h_0 \leq \epsilon_0$.
 Also $\forall h > h_0$, we have that
 $\bar{c}B_{N(\vx^*_h)}\bar{\gamma}^h \leq \epsilon_0$ and
 $\vx_{h}^* \in \calB(\vx^*, \rho)$.
\end{lemma}
\begin{proof}
 Since it is clear that BaMSOO would expand every node after a finite number
 of node expansions, we only have to show that there exists an $h_0$ that
 satisfies the conditions.
 By Lemma~\ref{sgpexp}, we have that
 $\forall h$ $B_{N(\vx^*_h)} = \mathcal{O}(\sqrt{h})$.
 Since $\bar{\gamma} < 1$, there exists an $h_0$ such that
 $\bar{c}B_{N(\vx^*_h)}\bar{\gamma}^h \leq \epsilon_0$ $\forall h > h_0$.
 Since $f(\vx^*_h) > f^* - \delta(h)
 > f^* - \bar{c}B_{N(\vx^*_h)}\bar{\gamma}^h \geq f^* - \epsilon_0$, we have by
 Assumption~\ref{envelop} that, $\vx^*_h \in \calB(\vx^*, \rho)$.
\end{proof}

\begin{lemma}
\label{bbelow}
 $\sum_{h=0}^{H} |I_h^{\bar{\delta}(H)}|
 \leq C \left(B_{N\left(\vx^*_H\right)}\right)^{D/2} \gamma^{(D/4-D/\alpha)H}$
 for some constant $C$
 for all $H > h_0$.
\end{lemma}
\begin{proof}
 By Lemma~\ref{fintime}, we know that
 $\bar{\delta}(H)
 = \bar{c}B_{N(\vx^*_H)}\bar{\gamma}^H
 < \epsilon_0$ if $H> h_0$.
 Therefore, by Assumption~\ref{envelop}, we have that
 $\chi_{\bar{\delta}(H)} = \{\vx \in \chi : f(\vx) \geq f^*
 - \bar{\delta}(H)\} \subseteq \calB(x^*, \rho).$
 Again by Assumption~\ref{envelop}, we have that
 $$
 f^* - \bar{\delta}(H) \leq f(\vx) \leq f^* - c_2\|\vx - \vx^*\|^2_2
 \mbox{  } \forall \vx \in \chi_{\bar{\delta}(H)}.
 $$
 Thus $\chi_{\bar{\delta}(H)} \subseteq
 \calB \left(x^*, \sqrt{\frac{ \bar{\delta}(H)}{c_2}}\right)
 = \calB \left(x^*, \sqrt{\frac{ \bar{c}B_{N\left(\vx^*_H\right)}\gamma^{H/2}}{c_2}}\right)$.

 Since each cell $(h, i)$ contains a $\ell$-ball of radius $\nu \delta(h)$ centered at
 $\vx_{h, i}$ we have that each cell contains a ball
 $\calB(\vx_{h, i}, (\nu \delta(h))^{1/\alpha})
 = \calB(\vx_{h, i}, (\frac{\nu c}{c_1})^{1/\alpha} \gamma^{h/\alpha})$.
 By the argument of volume, we have that
 $|I_h^{\bar{\delta}(H)}| \leq
 C_1 \left(B_{N\left(\vx^*_H\right)}\right)^{D/2} \gamma^{HD/4 - hD/\alpha}$
 for some constant $C_1$.
 Finally,
 \begin{eqnarray*}
  \sum_{h=0}^{H} |I_h^{\bar{\delta}(H)}| &\leq&
  C_1 \sum_{h=0}^{H} \left(B_{N\left(\vx^*_H\right)}\right)^{D/2} \gamma^{HD/4 - hD/\alpha} \\
  &=& C_1 \left(B_{N\left(\vx^*_H\right)}\right)^{D/2} \gamma^{HD/4}
  \sum_{h=0}^{H}  \gamma^{-hD/\alpha} \\
  &=& C_1 \left(B_{N\left(\vx^*_H\right)}\right)^{D/2} \gamma^{HD/4}
  \sum_{h=0}^{H}  \left(\gamma^{D/\alpha}\right)^{h-H} \\
  &\leq& C_1 \left(B_{N\left(\vx^*_H\right)}\right)^{D/2} \gamma^{HD/4}
  \sum_{h=0}^{\infty}  \left(\gamma^{D/\alpha}\right)^{h-H} \\
  &=& C_1 \left(B_{N\left(\vx^*_H\right)}\right)^{D/2} \gamma^{HD/4}
  \frac{\gamma^{-DH/\alpha}}{1- \gamma^{D/\alpha}} \\
  &=& \frac{C_1}{1- \gamma^{D/\alpha}}
  \left(B_{N\left(\vx^*_H\right)}\right)^{D/2} \gamma^{HD/4-DH/\alpha}\\
  &=& \frac{C_1}{1- \gamma^{D/\alpha}}
  \left(B_{N\left(\vx^*_H\right)}\right)^{D/2} \gamma^{(D/4-D/\alpha)H}.
 \end{eqnarray*}
 Setting $C = \frac{C_1}{1- \gamma^{D/\alpha}}$ gives us the desired result.
\end{proof}

\begin{lemma}
 \label{opt}
 Suppose $\calL(\vx^*_h) \leq f(\vx^*_h) \leq \mathcal{U}(\vx^*_h)$.
 If $\vx_{h}^*$ is not evaluated (that is $\mathcal{U}(\vx^*_h) < f^+$)
 then $f^+$ is $\delta(h)$-optimal.
\end{lemma}
\begin{proof}
 $f^+ > \mathcal{U}(\vx^*_h) \geq f(\vx^*_h) > f^*-\delta(h).$
\end{proof}

\subsection{Main Results}
\subsubsection{Simple Regret}
Let $h^*_n$ be the deepest level of an expanded optimal node with $n$ node
expansions. This following lemma is adapted from Lemma 2 of~\cite{Munos:2011}.
\begin{lemma}
 \label{lem:induction}
 Suppose $\calL(\vx) \leq f(\vx) \leq \mathcal{U}(\vx)$
 for all $\vx$ whose confidence region are evaluated.
 Whenever $h\leq h_{\max}(n)$ and
 $n \geq C h_{\max}(n)
 \sum_{i=h_0}^{h} \left(B_{N\left(\vx^*_i\right)}\right)^{D/2}
 \gamma^{(D/4-D/\alpha)i} + n_0$
 for some constant $C$, we have $h^*_n \geq h$.
\end{lemma}
\begin{proof}
 We prove the statement by induction.
 By Lemma~\ref{fintime}, we have that after $n_0$ node expansions, a node
 $\vx_{h_0}^* \in \calB(\vx^*, \rho)$ is expanded.
 Also $\forall h > h_0$, we have that
 $\bar{c}B_{N(\vx^*_h)}\bar{\gamma}^h \leq \epsilon_0$ and
 $\vx_{h}^* \in \calB(\vx^*, \rho)$.
 For $h = h_0$, the statement is trivially satisfied.
 Thus assume that the statement is true for $h$.
 Let n be such that
 $n \geq C h_{\max}(n)
 \sum_{i=h_0}^{h+1} \left(B_{N\left(\vx^*_i\right)}\right)^{D/2} \gamma^{(D/4-D/\alpha)i}
 + n_0$.
 By the inductive hypothesis we have that $h_n^* \geq h$.
 Assume $h_n^* = h$ since otherwise the proof is finished.
 As long as the optimal node at level $h+1$ is not expanded,
 all nodes expanded at the level are
 $\bar{\delta}(h+1)$-optimal by Lemma~\ref{3dopt}.
 By Lemma~\ref{bbelow}, we know that after
 $C h_{\max}(n)  \left(B_{N\left(\vx^*_{h+1}\right)}\right)^{D/2}
 \gamma^{(D/4-D/\alpha)(h+1)}$
 node expansions,
 the optimal node at level $h + 1$ will be expanded
 since there are at most $\sum_{i=0}^{h+1} \left|I_i^{\bar{\delta}(h+1)}\right|$
 $\bar{\delta}(h+1)$-optimal nodes at or beneath level $h+1$.
 Thus $h_n^* \geq h + 1$.
\end{proof}

\begin{theorem}
 \label{keytheorem}
 Suppose $\calL(\vx) \leq f(\vx) \leq \mathcal{U}(\vx)$
 for all $\vx$ whose confidence region is evaluated.
 Let us write $h(n)$ to be the smallest integer $h \geq h_0$ such that
 $$C h_{\max}(n)
 \sum_{i=h_0}^{h} \left(B_{N\left(\vx^*_i\right)}\right)^{D/2}
 \gamma^{(D/4-D/\alpha)i} + n_0 \geq n.$$
 Then the loss is bounded as
 $$r_n \leq \delta(\min\{h(n), h_{\max}(n) + 1\})$$
 and $h_n^* \geq \min\{h(n) -1, h_{\max}(n)\}$.
\end{theorem}
\begin{proof}
 From Lemma~\ref{bbelow}, and the definition of $h(n)$ we have that
 \begin{eqnarray*}
 C h_{\max}(n)
 \sum_{i=h_0}^{h(n)-1} \left(B_{N\left(\vx^*_i\right)}\right)^{D/2}
 \gamma^{(D/4-D/\alpha)i} + n_0
 < n.
 \end{eqnarray*}
 By Lemma~\ref{lem:induction}, we have that $h_n^* \geq h(n) -1$
 if $h(n) - 1 \leq h_{\max}(n)$ and $h_n^* \geq h_{\max}(n)$ otherwise.
 Therefore $h_n^* \geq \min\{h(n) -1, h_{\max}(n)\}$.

 By Lemma~\ref{opt}, we know that if $\vx_{h_n^*+1}^*$ is not evaluated then
 $f^+$ is $\delta(h_n^*+1)$-optimal. If $\vx_{h_n^*+1}^*$ is evaluated,
 then $f\left(\vx_{h_n^*+1}^*\right)$ is $\delta(h_n^*+1)$-optimal. Thus
 $r_n \leq \delta(\min\{h(n), h_{\max}(n) + 1\}).$
\end{proof}

\begin{proof}[Proof of Corollary~\ref{simpleReg}]
 Suppose $\calL(\vx) \leq f(\vx) \leq \mathcal{U}(\vx)$
 for all $\vx$ whose confidence region is evaluated.
 By Lemma~\ref{sgpbound}, we know that this holds with probability at
 least $1-\eta$.

 By the definition of $h(n)$ we have that
 \begin{eqnarray}
 \label{no5}
  n &\leq& C h_{\max}(n)
 \sum_{i=h_0}^{h(n)} \left(B_{N\left(\vx^*_i\right)}\right)^{D/2}
 \gamma^{(D/4-D/\alpha)i} + n_0 \nonumber \\
 &\leq& C h_{\max}(n) \left(B_{N\left(\vx_{h(n)}^*\right)}\right)^{D/2}
 \sum_{i=h_0}^{h(n)} \gamma^{-di} + n_0 \nonumber \\
 &\leq& C h_{\max}(n) \left(B_{N\left(\vx_{h(n)}^*\right)}\right)^{D/2}
 \gamma^{-dh_0}\frac{\gamma^{-dh(n)}-1}{\gamma^{-d}-1} + n_0
 \end{eqnarray}
 If $h(n) \leq h_{\max}(n)+1$, then by Theorem~\ref{keytheorem}, we have
 that $h_n^* \geq h(n) -1$. After $n$ expansions, the optimal node
 $\vx_{h(n)-1}^*$ has been expanded which suggests that its children's
 confidence bounds have been evaluated.
 Hence, $N\left(\vx_{h(n)}^*\right) < 2n$ since there have only been $n$ expansions.
 Therefore,
 $$(\ref{no5}) \leq K n^{\epsilon} \left(B_{2n}\right)^{D/2}
 \gamma^{-dh(n)}$$
 for some constant $K$ which implies that
 $$\gamma^{h(n)} \leq K^{1/d} B_{2n}^{\frac{2\alpha}{4-\alpha}} n^{-\frac{1-\epsilon}{d}}
 = K^{1/d} \left[2\log(4\pi^2n^2/{6\eta})\right]^{\frac{\alpha}{4-\alpha}}
 n^{-\frac{1-\epsilon}{d}}.$$
  By Theorem~\ref{keytheorem}, we have that
 $$r_n \leq c
 \min \left\{K^{1/d} \left[2\log(4\pi^2n^2/{6\eta})\right]^{\frac{\alpha}{4-\alpha}}
 n^{-\frac{1-\epsilon}{d}}, \gamma^{(n+1)^{\epsilon}} \right\}
 = \mathcal{O}\left(n^{-\frac{1-\epsilon}{d}}
 \log^{\frac{\alpha}{4-\alpha}}(n^2/\eta) \right).$$
\end{proof}



\end{document}